\documentclass[letterpaper]{article} 
\usepackage{aaai2026}  
\usepackage{times}  
\usepackage{helvet}  
\usepackage{courier}  
\usepackage[hyphens]{url}  
\usepackage{graphicx} 
\usepackage{natbib}  
\usepackage{caption} 
\usepackage{algorithm}
\usepackage{algorithmic}



\usepackage[english]{babel}

\usepackage{subcaption}
\usepackage{booktabs}
\usepackage{bibentry}
\usepackage{amsmath}
\usepackage{amsfonts}
\usepackage{amsthm}  

\newtheorem{assumption}{Assumption}
\newtheorem{lemma}{Lemma}
\newtheorem{proposition}{Proposition}
\newtheorem{definition}[section]{Definition}
\newtheorem{corollary}{Corollary}[section]

\usepackage{newfloat}
\usepackage{listings}
\DeclareCaptionStyle{ruled}{labelfont=normalfont,labelsep=colon,strut=off} 
\floatstyle{ruled}
\newfloat{listing}{tb}{lst}{}
\floatname{listing}{Listing}

\title{Learning Fair Representations with Kolmogorov-Arnold Networks}

\author{
    Amisha Priyadarshini\textsuperscript{\rm 1}, Sergio Gago-Masague\textsuperscript{\rm 1}
}
\affiliations{
    \textsuperscript{\rm 1}Department of Computer Science, University of California, Irvine, USA\\

    apriyad1@uci.edu, sgagomas@uci.edu
}

\usepackage{bibentry}

\begin{document}

\maketitle

\begin{abstract}
Despite recent advances in fairness-aware machine learning, predictive models often exhibit discriminatory behavior towards marginalized groups. Such unfairness might arise from biased training data, model design, or representational disparities across groups, posing significant challenges in high-stakes decision-making domains such as college admissions. While existing fair learning models aim to mitigate bias, achieving an optimal trade-off between fairness and accuracy remains a challenge. Moreover, the reliance on black-box models hinders interpretability, limiting their applicability in socially sensitive domains.
\\To circumvent these issues, we propose integrating Kolmogorov-Arnold Networks (KANs) within a fair adversarial learning framework. Leveraging the adversarial robustness and interpretability of KANs, our approach facilitates stable adversarial learning. We derive theoretical insights into the spline-based KAN architecture that ensure stability during adversarial optimization. Additionally, an adaptive fairness penalty update mechanism is proposed to strike a balance between fairness and accuracy. We back these findings with empirical evidence on two real-world admissions datasets, demonstrating the proposed framework's efficiency in achieving fairness across sensitive attributes while preserving predictive performance.


\end{abstract}


\section{Introduction}

In recent years, the widespread adoption of Machine Learning (ML) models in high-stakes decision-making domains, such as college admissions, healthcare, and hiring, has underscored the need for ethically aligned, fairness-aware Artificial Intelligence (AI) systems. While modern deep learning (DL) models offer high predictive capacity, they remain vulnerable to amplifying historical bias in real-world datasets \cite{bickel1975sex}. In this context, undergraduate college admissions have undergone notable transformations, emphasizing the need for fairness and equal opportunity for all applicants. Recent shifts in admission policies, such as the University of California's Non-Discriminatory Policy \cite{nondp}, strive towards non-discriminatory review process with elimination of standardized testing. Although these approaches represent definite strides towards an equitable system, the socioeconomically marginalized groups continue to face systemic barriers in the admissions process \cite{chetty2023diversifying}. The disparity in resources and support available to these groups often result in unequal academic performance. This underscores the critical need for robust algorithmic approaches that ensure a balanced predictive performance while maintaining fairness across protected attributes to prevent disparate outcomes.
\\Various fair learning frameworks have been proposed in the past to address these concerns in ML models \cite{priyadarshini2024fair}, \cite{petrovic2022fair}. These include pre-processing, in-processing, and post-processing techniques aimed at mitigating bias \cite{mehrabi2021survey}. Given the sensitive nature of real-world datasets, particularly in the domain of college admissions where the broad spectrum of input features contributes towards a holistic evaluation \cite{holistic_eval}, choosing a fairness mechanism is a challenge. The pre-processing approaches modify the input distributions, and hence are unsuitable for the case. Similarly, post-processing techniques have a limited direct influence over the model's internal representations, which curbs their effectiveness and raises fairness concerns. This further emphasizes the need for in-processing methods, such as adversarial debiasing \cite{zhang2018mitigating}, that enforce fairness constraints directly within the training process, while preserving the integrity of the data. In this paper, we examine adversarial debiasing as a means to mitigate socioeconomic bias in our college admissions decision-making process.
\\To meet the goals of adversarial robustness and model interpretability, we adopt Kolmogorov-Arnold Networks (KANs) \cite{liu2024kan} - a novel architecture grounded in Kolmogorov-Arnold (KA) representation theorem \cite{schmidt2021kolmogorov}. Unlike traditional Multi-Layer Perceptron (MLP) models, which employ scalar weights and fixed activation functions, KANs leverage a composition of learnable univariate spline functions. This allows for a flexible function approximation with inherent model interpretability and smoothness. The edge-based spline parameterization offers fine-grained control over feature transformations, making them well-suited for structured tabular domains. In this work, we propose an adversarially trained KAN framework with an adaptive $\lambda$ update policy, aimed at balancing fairness metrics and predictive scores. We validate this framework across two different sets of real-world freshman applicants data to the University of California Irvine (UCI). Theoretical analysis and numerical experiments further demonstrate the efficiency of our proposed framework in achieving robust performance under fairness constraints while outperforming state-of-the-art (SOTA) baseline models.

\section{Problem Formulation}

 In the context of college admissions, several studies, \cite{bickel1975sex}, \cite{bhattacharya2017university}, depict the need for fairness-aware decision-making. Given the profound societal impact, the admissions process must uphold equitable treatment across socio-demographic groups. Historical bias embedded in the data often results in unfair outcomes, disproportionately impacting the marginalized communities \cite{gandara2024inside}. To systematically analyze such disparities, we formalize the admission decision task into a binary classification problem.
 \\In our formulation, the model is trained using $m$ examples $(x_i, z_i, y_i)_{i=1}^m$, where each of these is composed of a feature vector $x_i \in \mathbb{R}^n$, containing $n$ predictors, a binary sensitive attribute $z_i$, and a binary label $y_i$. These examples are sampled from a training distribution, say, $\Gamma = (X, Z, Y) \sim s$. Let $\mathcal{X}$, and $\mathcal{Y}$ be the feature space and label space respectively. 
 We aim to develop a framework that leverages KANs within an adversarial debiasing setup using an adaptive penalty update mechanism. The goal is to mitigate socioeconomic bias while maintaining predictive performance by effectively balancing the trade-off between fairness and accuracy. We try to approximate the predictor model, $f$, with parameters $\theta_f$ in a zero-sum game setup with an adversary, $g$, with parameters $\theta_g$. One of the key factors for algorithmic fairness is group fairness. For our case, we consider two types of fairness definitions: \textit{Demographic Parity and $p \%$-Rule}. In this section, we present the fairness notions and works that are necessary to ground our proposition.
     
       \subsection{Demographic Parity} 
       Demographic Parity (DP), also known as statistical parity, is satisfied when the model's predicted positive outcomes are equally distributed across different sensitive groups \cite{dwork2012fairness}. Formally for a binary sensitive attribute \( z \in \{0, 1\} \), DP requires the following conditions to hold:
       \begin{equation}
           \mathbb{E} \left[ \mathbb{I}(\hat{f}_{\theta_f}(x) = 1) \mid z = 0 \right] = \mathbb{E} \left[ \mathbb{I}(\hat{f}_{\theta_f}(x) = 1) \mid z = 1 \right]
       \end{equation}
        Here,  \( \hat{f}_{\theta_f}(x) = \mathbb{I}(f_{\theta_f}(x) > 0.5) \) denotes the predicted label after applying a decision threshold to the model output, and \( \mathbb{I}(\cdot) \) is the indicator function. Based on this criterion, the fairness-aware learning objective can be formulated as:
       \begin{equation}
       \begin{split}
           \arg\min_{\theta_f} \: \mathbb{E}_{(x,z,y)\sim s} \: \mathcal{L}_\mathcal{Y}(f_{\theta_f}(x), y)\\ 
           \text{s.t.} \quad |\mathbb{E}_s[\hat{f}_{\theta_f}(x)|z=1] - \mathbb{E}_s[\hat{f}_{\theta_f}(x)|z=0]| < \epsilon
        \end{split}
       \end{equation}
       where \( \epsilon \) is a small tolerance parameter that controls the allowable disparity between groups.




       \subsection{p$\% $-Rule}
       The \( p\% \)-Rule is a widely used group fairness metric that quantifies the disparity in the rates of positive outcomes across different demographic groups \cite{hardt2016equality}. For a binary sensitive attribute \( z \in \{0, 1\} \), the predicted positive outcome rates for both the groups can be defined as, $r_0 = \mathbb{E} \left[ \mathbb{I}(\hat{f}_{\theta_f}(x) = 1) \mid z = 0 \right]$, $r_1 = \mathbb{E} \left[ \mathbb{I}(\hat{f}_{\theta_f}(x) = 1) \mid z = 1 \right]$. 
       Then, the \( p\% \)-Rule requires:
        \begin{equation}
            \frac{\min(r_0, r_1)}{\max(r_0, r_1)} \geq \frac{p}{100}
        \end{equation}
        \\A higher value of the \( p\% \)-Rule (closer to 1 or 100$\%$) indicates improved fairness, as it implies similar treatment across groups. 

\subsection{Limitations of Existing Approaches}
In standard supervised learning, a neural network (NN) model, \( f\), is typically trained by minimizing the expected prediction loss, $\mathcal{L}_{\mathcal{Y}}$,
without incorporating any fairness constraints. When applied to real-world datasets that reflect historical bias, such models can inadvertently encode and even amplify disparities in predicted outcomes across sensitive groups. This violates principles of group fairness, particularly when the sensitive attribute, \( z \), is correlated with the outcome, \( y \).
\\Adversarial debiasing framework extends this formulation by introducing an adversary \( g_{\theta_g} \) that attempts to infer the sensitive attribute \( z \) from the predictor's output \( f_{\theta_f}(x) \). The predictor \( f_{\theta_f} \) is simultaneously trained to minimize classification loss while maximizing the adversary’s prediction error, resulting in the following optimization objective: 
\begin{equation}
\begin{split}
    \min_{\theta_f} \: \mathbb{E}_{(x, z, y) \sim s} \: \mathcal{L}_{\mathcal{Y}}(f_{\theta_f}(x), y) \quad \\
    \text{s.t.}\quad \min_{\theta_g} \: \mathbb{E}_{(x, z, y) \sim s} \: \mathcal{L}_{\mathcal{Z}}(g_{\theta_g}(f_{\theta_f}(x)), z) > \epsilon'
    \end{split}
\end{equation}
where, $\epsilon'$ denotes the fairness threshold, and \( \mathcal{L}_{\mathcal{Z}} \) represents the adversary loss.
To obtain a better balance between the predictions of the predictor and adversary, \cite{zhang2018mitigating} proposes a more relaxed formulation incorporating a fairness coefficient, $\lambda$:
\begin{equation}\label{EQ1-adv_deb}
    \begin{split}
    \min_{\theta_f}\max_{\theta_g}\mathbb{E}_{(x, z, y) \sim s} \: \mathcal{L}_{\mathcal{Y}}(f_{\theta_f}(x), y) \\
    - \lambda \cdot \mathbb{E}_{(x, z, y)\sim s} \: \mathcal{L}_{\mathcal{Z}}(g_{\theta_g}(f_{\theta_f}(x)), z) 
    \end{split}
\end{equation}
\begin{figure*}[t]
    \centering
    \includegraphics[width=\textwidth]{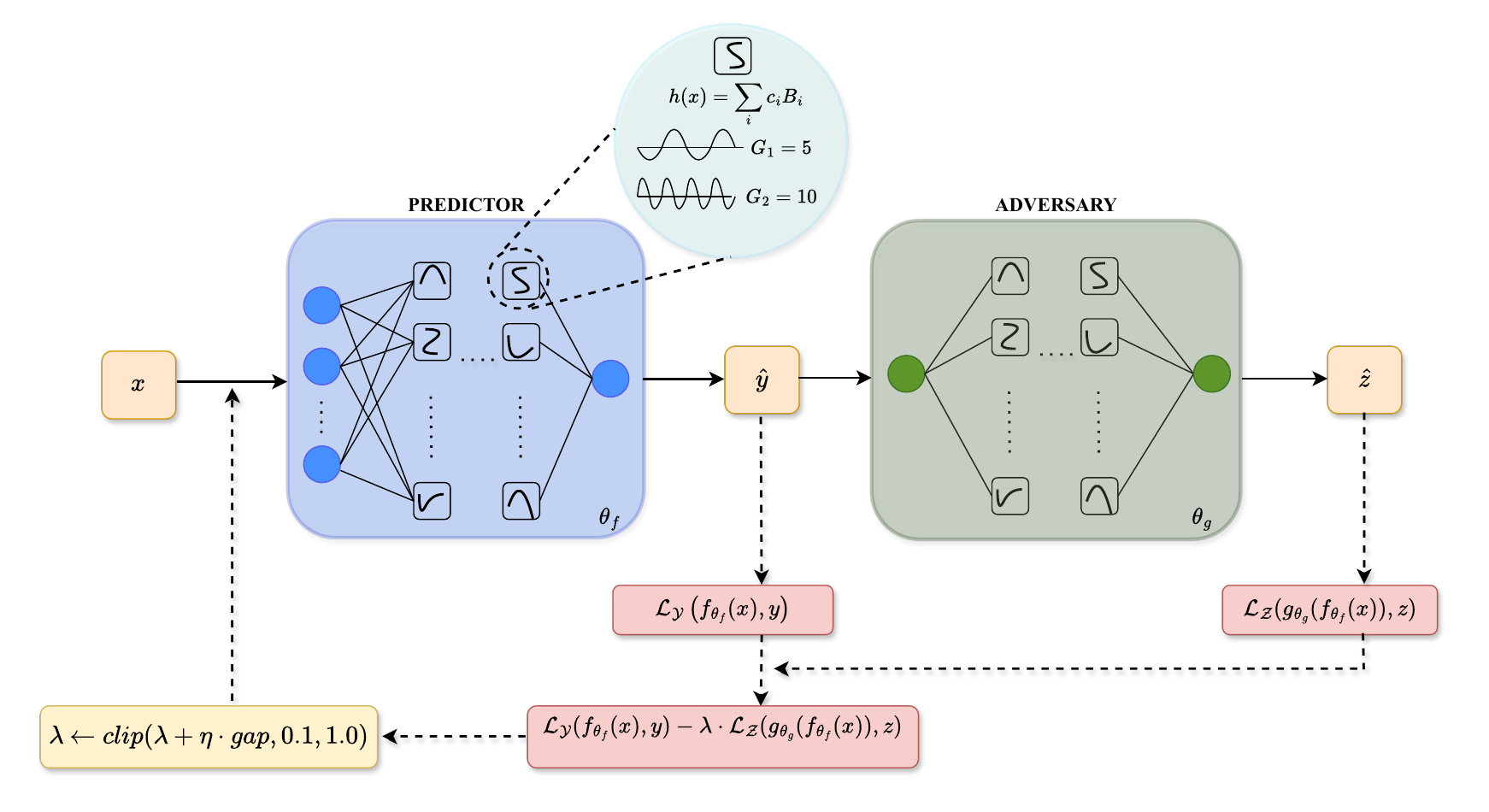}
    \caption{Schematic overview of the proposed adversarial debiasing framework using KANs. The classifier (KAN), $f$, learns predictive representations from input features, $x$, while an adversary (also a KAN), $g$, attempts to infer the sensitive attribute, $z$, from the classifier’s output, $y$, using a min-max objective. Fairness parameter, $\lambda$, is updated adaptively after every training epoch in an attempt to balance accuracy and fairness scores.}
    \label{fig:kan_framework}
\end{figure*}
\\where, \( \lambda \in \mathbb{R}^+ \) controls the degree of fairness for striking a balance between the tradeoff. 
\\While adversarial debiasing has proven effective in reducing group-level disparities, we discuss two persistent challenges that remain.
\textit{Firstly}, fairness constraints (such as DP) lead to provable loss in accuracy, especially when base rates differ across sensitive groups. According to \cite{zhao2022inherent}, any predictor that satisfies exact DP incurs an average classification error that is lower bounded by half the base rate gap between the groups. This formalizes the inherent fairness–accuracy tradeoff that can manifest as instability during model training. Moreover, enforcing strict criteria like DP can make models \textbf{less robust} to data perturbations or distribution shifts. The requirement to equalize output distributions across groups may lead to unreliable models, particularly in high-stake settings. \textit{Secondly}, while effective at mitigating group-level bias, the adversarial formulation suffers from a critical limitation: the learned neural predictor \( f_{\theta_f} \) remains a black-box function that lacks interpretability. Standard NNs do not provide explicit functional representations of feature-wise contributions, making it challenging to analyze decision logic, particularly in sensitive real-world settings such as college admissions.
\\In contrast, KAN's spline-based architecture and explicit functional decomposition offer a structured, interpretable, and robust alternative to traditional NN. Moreover, KAN retains the flexibility necessary for learning fair decision boundaries. In this study, we present a theoretical analysis and experimental backing demonstrating the suitability of KAN representations for adversarial learning setup.


\section{Related Work}

Several works, \cite{chetty2023diversifying}, \cite{barocas2023fairness}, \cite{zimdars2010fairness}, \cite{woo2023bias}, showcase the critical need for fair learning in the college admissions setup. Although various works like \cite{doleck2020predictive}, \cite{waters2014grade} demonstrate ML implementation in this context, adhering to fairness criteria or prioritizing the socioeconomically marginalized groups in the decision-making is still overlooked. This shortcoming is furthered by the reliance on DL models that are hard to interpret, leading to trust issues \cite{rudin2019stop}. In the context of college admissions, interpretability plays a significant role, given the direct impact of these decisions on applicants and their subject to public scrutiny.
\\Although fair-learning approaches like \cite{zhang2018mitigating}, \cite{madras2018learning}, \cite{lahoti2020fairness}, \cite{priyadarshini2025fairfusion}, implement adversarial learning to incorporate fairness constraints, they are defined in the context of DL setups. Even several mathematically grounded works such as, \cite{jia2024aligning} introducing Lipschitz-based fairness criterion, and \cite{chzhen2020fair} that leverage optimal mass transport for regression tasks, provide rigorous formulations but are mostly tailored to specific problem settings, limiting their applicability. Secondly, works like \cite{zhao2022inherent}, \cite{calders2009building}, \cite{kamiran2009classifying} further affirm the inherent fairness-accuracy tradeoffs, cited as one of the major problems in our case study. These underlying challenges motivate our proposed approach. While various works like \cite{kiamari2024gkan}, \cite{bodner2024convolutional}, \cite{vaca2024kolmogorov}, \cite{alter2024robustness} have explored KANs on various domains, to the best of our knowledge, our proposed work is the first to elaborate on KAN-based adversarial debiasing framework.


\section{Robust Adversarial Learning with KANs}
The Adversarial Debiasing \citep{zhang2018mitigating} framework adopts a zero-sum game approach, as shown in Eq. \ref{EQ1-adv_deb}. However, due to the inherent fairness-accuracy tradeoff, as discussed in \cite{zhao2022inherent}, improvements in fairness often comes at the cost of predictive performance. In particular, \cite{hardt2016equality} shows that enforcing DP can eliminate an otherwise perfect predictor, underscoring the importance of theoretically quantifying this tradeoff in classification tasks. Although several adversarial debiasing techniques based on deep NN models \cite{priyadarshini2024fair} have been proposed they often suffer from lack of interpretability. In contrast, we propose utilizing KAN as the learning model in an adversarial debiasing framework. Having demonstrated adversarial robustness \cite{alter2024robustness}, and structural interpretability, KAN models have a definite edge over MLPs. We further incorporate an adaptive $\lambda$-update policy in the KAN-based adversarial framework with various optimizer settings. This helps enhance the model’s robustness to fairness-accuracy trade-off. Despite the high complexity of the KAN representations, we further affirm their efficacy at generalization without overfitting to fairness constraints. An overview of the proposed adversarial learning framework is presented in Figure \ref{fig:kan_framework}.

\subsection{Preliminaries}

The KANs \cite{liu2024kan} are inspired by the KA Representation theorem \cite{kolmogorov1957representations}, which states that any multivariate continuous function can be written as a finite superposition of continuous univariate functions. A distinctive feature of these networks is their use of learnable univariate spline functions in place of the scalar weights and fixed activation functions of traditional MLPs. These characteristics improve the model's capacity for accurate and flexible approximations. A traditional KAN function approximator, $f:[0,1]^n \to \mathbb{R}$, can be expressed as:
\begin{equation}\label{eq:KAN}
    f(x_1, x_2, ..., x_n) = \sum_{q=1}^{2n+1} \Phi_q (\sum_{r=1}^{n} \phi_{q, r}(x_r))
\end{equation}
where the inner and outer univariate functions map as, $\phi_{q,r} :[0,1] \to \mathbb{R}$ and $\Phi_q : \mathbb{R} \to \mathbb{R}$ respectively. While the early critiques questioned the practical utility of the KA theorem due to the non-smoothness of inner functions \cite{girosi1989representation}, \cite{liu2024kan} demonstrate its viability in ML by leveraging smoother, sparser compositional structures, extending the network to a deep KAN architecture \cite{lin2017does}.
\\Given, all the functions to be learned are univariate, represented in Eq. \ref{eq:KAN}, each 1D function, $\phi_i$, is parametrized as a B-spline curve, say $h(x)$. It can be defined as $ h(x)= \sum_i c_i B_i(x)$, where $c_i$s are the learnable coefficients of the local B-spline basis functions. Each spline function is a piecewise-polynomial (typically cubic) function, and is $(k+1)$ (or, twice differentiable for cubic splines) times continuously differentiable \cite{girosi1989SPLINE}. In the following section, we discuss the structural properties of KANs that make them well-suited for fairness-aware adversarial learning with theoretical justifications.

\begin{assumption}\label{assumptions}
\textnormal{Given $\mathcal{X} \subset \mathbb{R}^n$ denotes the input feature space, and $Z \in \{0, 1\}$ represents the binary sensitive attribute, we impose the following conditions:}

    \begin{enumerate}
        \item \textnormal{The input space $\mathcal{X}$ is compact; that is, it is bounded and closed.}
        
        \item \textnormal{Each univariate spline function, $h(x)$, used in the KAN architecture is piecewise polynomial (usually cubic), and is twice continuously differentiable on its domain.}
        
    \end{enumerate}

\end{assumption}

\subsection{Theoretical Analysis}
    \begin{lemma}\label{lemma1}
        Each univariate spline function is Lipschitz continuous on a bounded range, is differentiable, and defined over a compact interval. Hence, $f$ is Lipschitz continuous on a bounded domain.
    \end{lemma}
    \begin{proof}
    Given $h_i : \mathbb{R} \to \mathbb{R}$ denotes the univariate spline functions used in KAN architecture. Based on our Assumption \ref{assumptions}, each $h_i$ is a cubic spline continuous function that is $C^2$ smooth (i.e., twice differentiable), and defined over a compact interval ($\subset \mathbb{R}$). Hence, by the Mean Value Theorem, there exists a constant $L_i > 0$ such that:
    $$|h_i(x) - h_i(x')| \leq L_i . |x - x'|$$
    for all $x, x' \in \mathbb{R}$. As the overall KAN function, $f$, is a finite summation and composition of these univariate spline functions, it follows that $f$ itself is $L$-Lipschitz continuous on a bounded input domain \cite{rudin1976principles}. 
    \end{proof}
    
    \begin{algorithm}[t]
    \caption{Adversarial training with adaptive $\lambda$ policy}
    \label{alg:adaptive-kan}
    
    \textbf{Input}: Initialize KAN-based classifier $f$ and adversary $g$ with grid $\mathcal{G}$ and order $k$. Training data $(x, z, y)$, classifier learning rate $\eta_{\text{clf}}$, adversary learning rate $\eta_{\text{adv}}$, threshold $\tau$, fairness learning rate $\eta$.\\
    \textbf{Parameter}: Initialize fairness penalty weight $\lambda$ 
    \begin{algorithmic}[1]
    \FOR{each grid level $\mathcal{G}_i \in \mathcal{G}$}
        \IF{$i = 0$}
            \STATE Initialize $f$ and $g$ on grid $\mathcal{G}_0$
        \ELSE
            \STATE Initialize $f$ and $g$ from trained models on $\mathcal{G}_{i-1}$
        \ENDIF
        \STATE Freeze classifier and compute outputs $f(x)$
        \STATE \textbf{Train adversary:}
        \STATE \hspace{1em} Update $g$ to minimize $\mathcal{L}_\mathcal{Z}(g(f(x)), z)$
        \STATE \textbf{Train classifier with adversarial debiasing:}
        \FOR{each epoch $t = 1, \dots, T$}
            \STATE Update $f$ to minimize:
            \[
            \mathcal{L}_\mathcal{Y}(f(x), y) - \lambda \cdot \mathcal{L}_\mathcal{Z}(g(f(x)), z)
            \]
            \STATE Evaluate fairness metric: compute $p\%\text{-Rule}$
            \STATE Compute gap: $\delta \gets (\tau - p\%\text{-Rule})/\tau$
            \STATE Update penalty: $\lambda \gets \text{clip}(\lambda + \eta \cdot \delta, 0.1, 1.0)$
        \ENDFOR
    \ENDFOR
    \end{algorithmic}
    \end{algorithm}
    
    \begin{lemma}\label{lemma2}
        Every univariate spline function, in KAN is $\beta_i$-smooth, and the finite sum of spline functions preserves the smoothness. Hence, $f$ is $\beta$-smooth.
    \end{lemma}
    \begin{proof}
    Given, $h_i$ is twice differentiable implies $|h_i''(x)|\le \beta_i$ for all $x \in \mathcal{X}$, for some $\beta_i \in [0,\infty)$. This proves that $h_i$ is $\beta_i$-smooth. Provided the overall KAN function, $f$, is constructed by combining these univariate spline functions through finite summations and compositions, we note two important properties: \textit{firstly}, the finite sum of spline functions preserves smoothness, and \textit{secondly}, the composition of smooth functions preserves smoothness \cite{rudin1976principles}. Hence, it implies that $f$ is $\beta$-smooth. A full justification of the said properties has been provided in the Appendix.
    \end{proof}
    As shown in Lemma \ref{lemma2}, the KAN function, $f$, inherits the smoothness of its constituent univariate functions, and remains $\beta$-smooth on the bounded domain, $\mathcal{X}$. This ensures stable convergence during optimization owing to bounded variations in gradients \cite{schmidt2021kolmogorov}. Furthermore, Lemma \ref{lemma1} proves Lipschitz continuity of KAN models, hence ensuring model robustness to small input perturbations \cite{alter2024robustness}. These structural properties play a crucial role in stabilizing the adversarial training dynamics and ensure effective debiasing. Further mathematical evidence suggesting the efficiency of KAN structure in an adversarial debiasing setup has been provided in the Appendix.



    \subsubsection{Adaptive Penalty Mechanism}

    One of the key challenges in an adversarial debiasing setup lies in balancing fairness and predictive performance, which is governed by the fairness coefficient $\lambda$. It scales the adversarial loss relative to the predictive loss. However, in a non-adaptive setup, determining $\lambda$ is non-trivial. The process requires extensive manual tuning that is computationally expensive \cite{sattigeri2019fairness}, especially given the high network complexity of KANs. To address these limitations, we introduce an adaptive $\lambda$ update mechanism that dynamically adjusts the fairness penalty during training. At each training epoch, we compute the fairness gap and update the penalty weights in proportion to this gap. The mechanism follows a linear rule with a learning rate, $\eta$, and is clipped within a safe range to ensure stability in trade-off as shown in Algorithm \ref{alg:adaptive-kan}. To further affirm our claims, we empirically demonstrate the effectiveness of our approach in the following section.


\section{Numerical Experiments}
\begin{table*}[t]
\centering

\begin{subtable}[t]{0.65\textwidth}
\centering
\caption{}
\label{tab1.1}
\renewcommand{\arraystretch}{0.9}
\begin{tabular}{l|l|c|c|c|c|c|c}
\toprule
\textbf{Model} & \textbf{Optimizer} & \textbf{Acc} & \textbf{AUROC} & \textbf{$p\%\text{-Rule}_{(1)}$} & \textbf{$p\%\text{-Rule}_{(2)}$} & \textbf{$\text{DP}_{(1)}$} & \textbf{$\text{DP}_{(2)}$} \\
\midrule
 & Adam   & 74.92 & 80.27 & 85.79 & 89.65 & 0.028 & 0.035 \\
$\text{K}_{(1)}$ & OAdam  & 74.47 & 76.86 & 90.39 & 92.99 & 0.072 & 0.084 \\
 & ADOPT  & \textbf{74.58} & \textbf{81.45} & \textbf{89.21} & \textbf{90.92} & \textbf{0.047} & \textbf{0.055} \\
\midrule
 & Adam   & 82.36 & 85.60 & 84.11 & 81.22 & 0.037 & 0.043 \\
$\text{K}_{(2)}$ & OAdam  & 81.69 & 86.26 & 99.25 & 99.55 & 0.026 & 0.03 \\
 & ADOPT  & \textbf{82.51} & \textbf{86.68} & \textbf{99.25} & \textbf{99.70} & \textbf{0.032} & \textbf{0.037} \\
 \midrule
 & Adam   & 72.55 & 72.53 & 98.32 & 99.61 & 0.009 & 0.011 \\
$\text{B}_{(1)}$ & OAdam  & 74.13 & 74.13 & 98.52 & 95.04 & 0.004 & 0.018 \\
 & ADOPT  & 70.19 & 70.20 & 97.15 & 93.50 & 0.019 & 0.022 \\
 \midrule
  & Adam   & 74.87 & 71.77 & 93.97 & 97.79 & 0.01 & 0.004 \\
$\text{B}_{(1)}'$ & OAdam  & 78.61 & 63.05 & 96.56 & 97.06 & 0.006 & 0.015 \\
 & ADOPT  & 77.75 & 63.45 & 90.64 & 99.0 & 0.006 & 0.004 \\
 \midrule
$\text{B}_{(2)}$ & Adam  & 72.44 & 76.12 & 95.22 & 97.59 & 0.024 & 0.012 \\ 
$\text{B}_{(2)}'$ & Adam  & 79.12 & 64.47 & 98.00 & 88.84 & 0.003 & 0.018 \\
\midrule
$\text{B}_{(3)}$ & Adam  & 74.92 & 80.95 & 83.10 & 83.91 & 0.085 & 0.08 \\
$\text{B}_{(3)}'$ & Adam  & 78.36 & 81.99 & 97.72 & 93.29 & 0.004 & 0.014 \\
\bottomrule
\end{tabular}
\end{subtable}
\hfill
\begin{subtable}[t]{0.27\textwidth}
\centering
\caption{}
\label{tab1.2}
\renewcommand{\arraystretch}{1.0}


\begin{tabular}{ll}
\toprule
$\text{K}_{(1)}$   & \begin{tabular}[t]{@{}l@{}}KAN (trained\\on $\text{D}_{(1)}$)\end{tabular} \\
$\text{K}_{(2)}$   & \begin{tabular}[t]{@{}l@{}}KAN (trained\\on $\text{D}_{(2)}$)\end{tabular} \\
$\text{B}_{(1)}$   & \begin{tabular}[t]{@{}l@{}}trained on\\$\text{D}_{(1)}$\end{tabular} \\
$\text{B}_{(1)}'$  & \begin{tabular}[t]{@{}l@{}}trained on\\$\text{D}_{(2)}$\end{tabular} \\
$\text{B}_{(2)}$   & \begin{tabular}[t]{@{}l@{}}trained on\\$\text{D}_{(1)}$\end{tabular} \\
$\text{B}_{(2)}'$  & \begin{tabular}[t]{@{}l@{}}trained on\\$\text{D}_{(2)}$\end{tabular} \\
$p\%\text{-Rule}_{(1)}$ &
  \begin{tabular}[t]{@{}l@{}}
    $p$\%-Rule\\
    (Low Inc.)
  \end{tabular} \\

$p\%\text{-Rule}_{(2)}$ &
  \begin{tabular}[t]{@{}l@{}}
    $p$\%-Rule\\
    (First Gen.)
  \end{tabular} \\

$\text{DP}_{(1)}$ &
  \begin{tabular}[t]{@{}l@{}}
    DP Gap\\
    (Low Inc.)
  \end{tabular} \\

$\text{DP}_{(2)}$ &
  \begin{tabular}[t]{@{}l@{}}
    DP Gap\\
    (First Gen.)
  \end{tabular} \\

\bottomrule
\end{tabular}

\end{subtable}

\caption{(a) Performance and fairness comparison across proposed frameworks, and baseline models, trained on two distinct datasets, $D_{(1)}$ and $D_{(2)}$, under fairness constraints with three different optimizers, Adam, OAdam, and ADOPT. (b) Notation reference for model identifiers and fairness metric used in (a).}
\label{tab1:kan_main_results}

\end{table*}

\subsection{Simulation Setup}

Various works, \cite{marcinkowski2020implications}, \cite{alvero2020ai}, have documented socioeconomic factors to significantly influence college admission outcomes. To empirically assess the efficacy of our proposed method in mitigating such bias, we conduct experiments on two distinct college admissions dataset collected from freshman applicants to the Dept. of Computer Science, UCI. The first dataset, denoted as $D_{(1)}$, comprises of 4,442 application records from the Fall 2021 admission cycle, and the second dataset, denoted as $D_{(2)}$, includes 26,243 application records spanning admission cycles from Fall 2018 to Fall 2024. Each of the datasets contain approximately 140 features, encompassing demographic, academic records, high school information, and essay question responses \cite{priyadarshini2023admission}. The broad spectrum of features facilitates \textit{holistic evaluation}, a comprehensive approach to college admissions that ensures applicants are assessed based on a wide range of criteria. For the study, after a systematic analysis of input features, we consider two binary sensitive attributes: low-income status and first-generation flag, which directly affect the socioeconomic demographics. We justify this choice by empirically demonstrating the presence of bias in model predictions using Kernel Density Estimation (KDE) plots, which are included in the Appendix. Furthermore, we use the final read score as the target variable for the binary classification task.
To quantitatively evaluate model fairness, we utilize two fairness metrics, $p\%$-Rule, and DP. In addition, to assess the predictive performance of our model, we employ Accuracy and AUROC scores. Next, we incorporate three different types of optimizers, namely Adam \cite{kingma2014adam}, Optimistic Adam (OAdam) \cite{daskalakis2017training}, and ADOPT \cite{taniguchi2024adopt}, for each training configuration. Despite KAN's structural robustness and theoretical guarantees in adversarial setting, we empirically explore different optimization strategies to enhance model convergence and stability during the debiasing process. Further discussions have been included in the Appendix.

\subsection{Baseline Models}
To evaluate the effectiveness of our proposed framework, we compare our approach against three baseline models: 
\begin{itemize}
    \item Baseline $B_{(1)}$ is a standard adversarial debiasing framework using a fully connected feedforward neural network (FFNN) as the classifier and adversary models \cite{priyadarshini2024fair},
    \item Baseline $B_{(2)}$ is a SOTA implementation of exponential gradient-based debiasing model \cite{agarwal2018reductions},
    \item Baseline $B_{(3)}$ is a SOTA method based on the ROAD (Robust Optimization for Adversarial Debiasing) framework \cite{grari2023fairnessROAD}.
\end{itemize}

\subsection{Simulation Results and Discussion}
In this section, we provide a comprehensive comparison of our proposed KAN-based adversarial debiasing framework with an adaptive $\lambda$ update mechanism across the two datasets, $D_{(1)}$ and $D_{(2)}$, using three distinct optimizer settings. In the experimental setup, as showcased in Table \ref{tab1:kan_main_results}, we use a fixed B-spline order of $k=3$. This spline order choice is consistent with \cite{liu2024kan}, that highlight it as striking an effective balance between expressiveness, smoothness, and training stability. Table \ref{tab1.1} reports the key performance metrics, including classification accuracy, AUROC, $p\%$-Rule, and DP across the two binary sensitive attributes, and Table \ref{tab1.2} lists the necessary notations.
\\On training the KAN models on $D_{(1)}$ dataset, we observe the Adam-based model to offer stronger predictive power (\textit{$0.45\%$ incr. in Acc., $3.41\%$ incr. in AUROC}) compared to OAdam, which enhances fairness outcomes significantly. However, the model trained using ADOPT exhibits a balanced performance across fairness and accuracy. On the other hand, training the KAN models on the larger $D_{(2)}$ dataset yields notable improvements in predictive performance and fairness. Although all setups trained on $D_{(2)}$ show significant performance improvement, the ADOPT-based model stands out by achieving balanced optimization between accuracy (\textit{$0.15\%$, and $0.82\%$ incr. in Acc. compared to Adam setup}) and fairness metrics (\textit{$\{15.14\%, 18.48\%\}$ incr. in $p\%$-Rule compared to Adam setup}). Meanwhile, the $B_{(1)}$ and $B_{(1)}'$ models under-perform relative to the KAN models, especially on the $D_{(2)}$ dataset. This suggests KAN-based adversarial learning to be more effective in maintaining fairness-accuracy trade-offs while capturing complex patterns, even in a feature-rich setting. SOTA baseline models, $B_{(2)}$ and $B_{(3)}$, display varying behavior. $B_{(2)}$ achieves high fairness on $D_{(1)}$ but suffers a drop on $D_{(2)}$, whereas $B_{(3)}$ starts with lower fairness on $D_{(1)}$ but improves substantially on $D_{(2)}$ with proper fine-tuning.

The empirical evidence demonstrates that the KAN-based adversarial debiasing framework, regardless of the choice of optimizer, consistently outperforms the baseline models across both the datasets. This can be largely attributed to the adaptive policy usage within the training module. Its ability to dynamically adjust and bound the fairness constraint, $\lambda$, achieves to secure a balanced performance. Amongst all the configurations, the model trained on ADOPT demonstrates comparative superior balanced performance. This observation suggests ADOPT to be well suited to the spline-based architecture of KANs, owing to their adaptive convergence behavior \cite{taniguchi2024adopt}. We highlight the frameworks that demonstrate a balanced fairness-accuracy trade-off in Table \ref{tab1:kan_main_results} for clarity.

\begin{figure*}[t]
    \centering
    \includegraphics[width=0.62\textwidth]{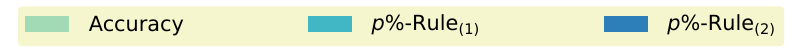}

    \begin{subfigure}[t]{0.55\textwidth}  
        \centering
        \includegraphics[width=\linewidth, height=3.2cm]{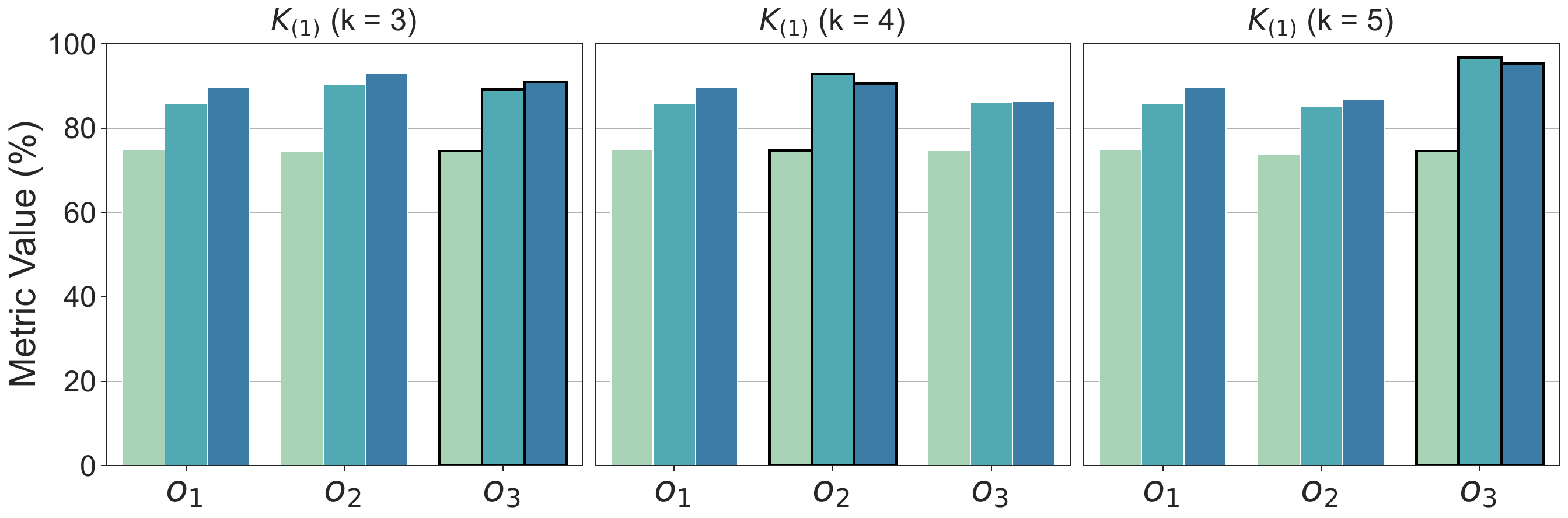}
        \caption{}
        \label{fig:sub_a}
    \end{subfigure}
    \hspace{0.005\textwidth}  
    \begin{subfigure}[t]{0.43\textwidth}  
        \centering
        \includegraphics[width=\linewidth, height=3.2cm]{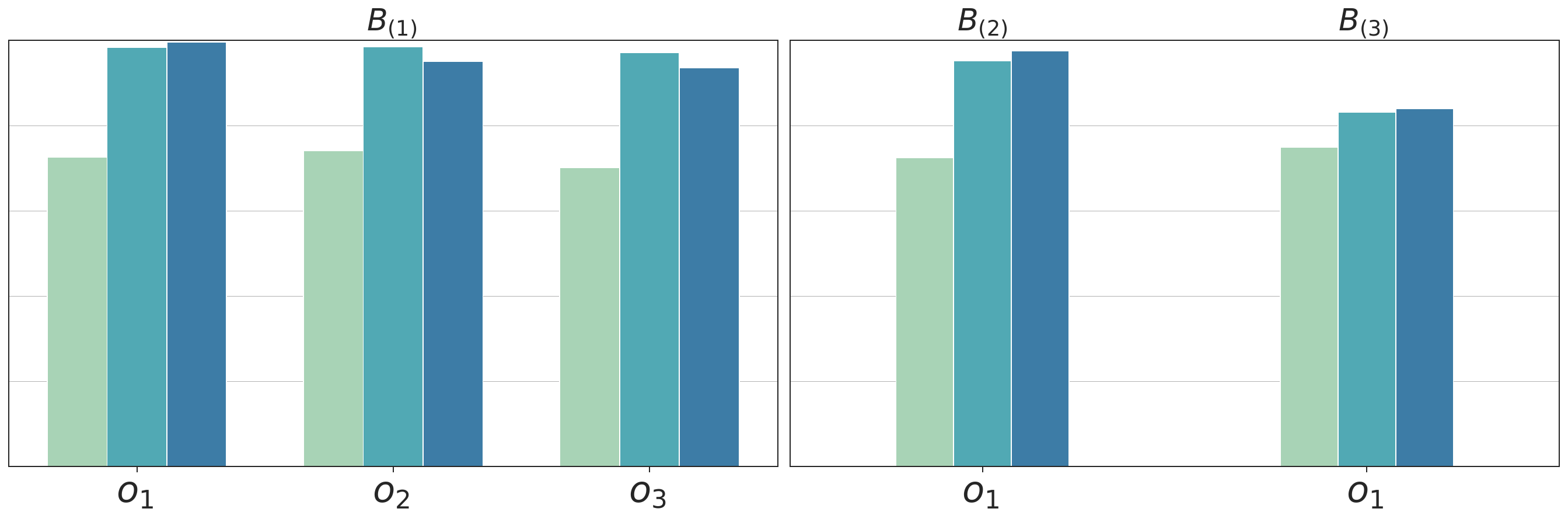}
        \caption{}
        \label{fig:sub_b}
    \end{subfigure}
    \caption{Ablation study to compare predictive performance and fairness scores across proposed frameworks with different spline knot complexities, and baseline models under varying optimization strategies, evaluated on dataset $D_{(1)}$. (a) Illustrates metrics of the proposed KAN-based adversarial frameworks using three different optimizer techniques ($o_1$: Adam, $o_2$: OAdam, $o_3$: ADOPT). Similarly, (b) depicts the Baseline model performance across varying optimizers.}
    \label{fig:side_by_side}
\end{figure*}

\begin{figure*}[t]
    \centering
    \hspace*{0.02\textwidth}
    \includegraphics[width=0.72\textwidth]{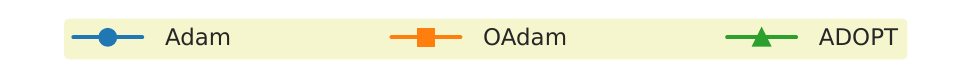}
        \includegraphics[width=\linewidth, height=4.7cm]{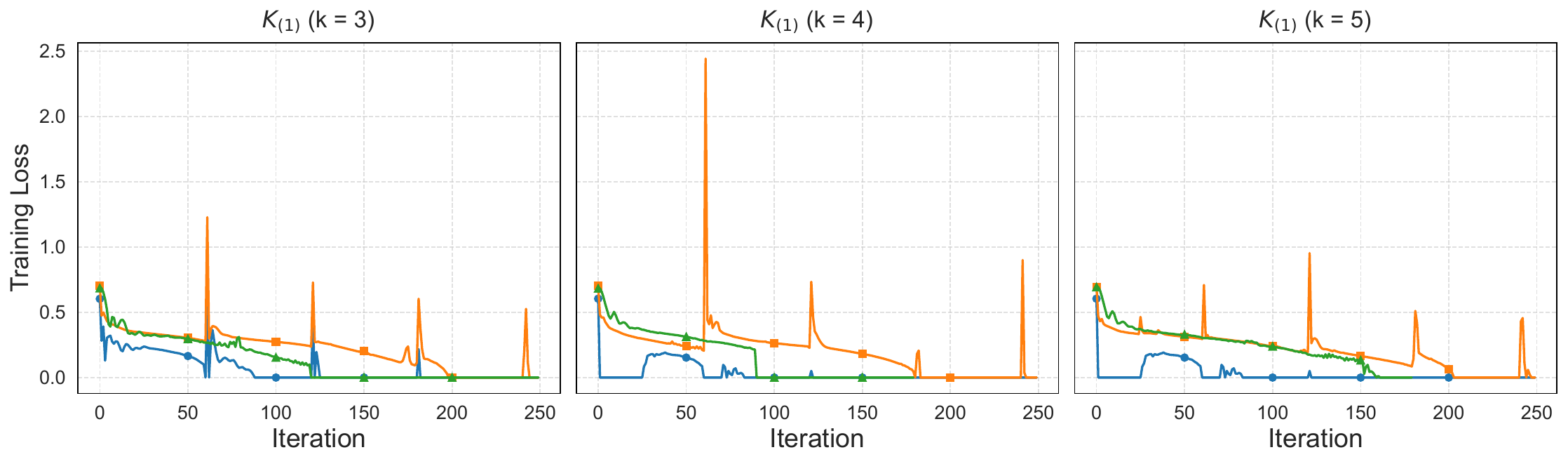}
        \caption{Training Loss convergence of the KAN-based adversarial learning framework across three different spline knot orders ($k \in \{3,4,5 \}$), under three different optimization techniques. Each plot illustrates the convergence behavior of the model, highlighting the impact of optimizer choice and KAN model's spline complexity on training stability.}
        \label{fig:loss vs itr}

\end{figure*}

\subsection{Ablation Study}
To investigate how KAN spline knot complexity, $k$, affects the adversarial learning process, we perform an ablation study focusing on three spline orders $k = \{3,4,5 \}$. We chose the $D_{(1)}$ dataset for the experiments due to its controlled small-scale setting. It enables focused assessment of the model's sensitivity to architectural changes under constrained data conditions. In addition to the KAN-based adversarial learning models, $K_{(1)}$, we also train the baseline model, $B_{(1)}$, across the three optimizer settings, and the SOTA baseline models, $B_{(2)}$, and $B_{(3)}$ across the Adam optimizer. As illustrated in Figure \ref{fig:sub_a}, increasing the spline order $k$ leads to subtle variations in accuracy and fairness metrics. Although the $K_{(1)}$ models exhibit relatively stable accuracy across different $k$ setups, the fairness metrics, measured via the $p\%\text{-Rule}$, tend to improve with higher spline orders. This suggests that more expressive spline functions, in an adversarial setting, can better model fair representations. Although higher $k$ (e.g., $k=4, 5$) may introduce convergence instability, as observed in Figure \ref{fig:loss vs itr}, the model trained with ADOPT exhibits comparatively stable training dynamics than those trained on Adam or OAdam. 
Next, when comparing the $K_{(1)}$ models with the baseline MLP models, $B_{(1)}$, we observe that the former consistently outperforms the latter in achieving a balance between accuracy and fairness. As observed in Figure \ref{fig:sub_b}, $B_{(1)}$ models have a higher fairness score but comparatively lower accuracy across the three optimizer settings. This further empirically substantiates the efficiency of the adaptive $\lambda$ update policy. The SOTA baseline models, $B_{(2)}$ and $B_{(3)}$, show strong performance in fairness but incur slight drops in accuracy as seen in Figure \ref{fig:sub_b}. As seen in Figure \ref{fig:loss vs itr}, the $K_{(1)}$ model training loss curve exhibits pronounced spikes. This could be attributed to the smoothness-dimension trade-off \cite{samadi2024smooth} in KANs. While this can induce oscillatory behavior during training, it is inherent to preserving the universal approximation. A formal discussion of this property is provided in the Appendix. Despite the fluctuations in KAN-based setups, our findings help us empirically establish the efficiency of the proposed adversarial learning framework.

\section{Conclusion and Future Work}
In this study, we propose a KAN-based adversarial learning framework with adaptive $\lambda$ update policy to mitigate socioeconomic bias in a college admission setting. We discuss the various structural properties of KAN, providing proper mathematical justification behind the adversarial robustness of the networks. Through extensive experiments on two distinct real-world admissions datasets, we demonstrate the efficiency of our proposed framework by consistently outperforming SOTA baselines while achieving a balance between the inherent accuracy-fairness tradeoff.
While our proposed approach effectively mitigates output-level bias, it opens up several promising enhancements. The incorporation of explicit feature-level bias detection through a reweighting mechanism could improve the overall fairness of the model while also helping to adhere to any data privacy constraints. Additionally, extending the current framework by incorporating alternative fairness definitions beyond DP could provide new insights.
We believe this work offers a promising step towards aligning algorithmic performance with social responsibility, encouraging future explorations into KAN-based fairness-aware systems that could uphold both ethical and performance standards in critical decision-making.


\section{Acknowledgments}
We gratefully acknowledge the University of California, Irvine, for providing access to the admissions datasets used in this study. We also thank Prof. Iftekhar Ahmed for providing access to GPU resources used for large-scale experiments.

\bibliography{aaai2026}

\clearpage
\onecolumn
\appendix

\section{Additional Theoretical Analysis}

\subsection{KAN Structural Validations}
In the subsection of Theoretical Analysis we propose that the KAN function, $f$, is L-Lipschitz continuous and $\beta$-smooth. Let Assumption \ref{assumptions} hold. Then, we provide supporting proofs in assertion of Lemma \ref{lemma2} that states KAN models, $f$, inherit the $\beta$-smoothness from its corresponding univariate spline functions. 
\begin{proposition}\label{prop 1}
    Finite sum of spline functions preserves the smoothness.
\end{proposition}
\begin{proof}
    As provided in Eq. \ref{eq:KAN}, $f$ is composed of inner functions $\phi$ ($=h(.)$), and outer functions, $\psi$, where each of them is a smooth spline. Let the inner sum be defined as , $v_k(x) = \sum_{i=1}^n \phi_{k_i}(x_i)$ Then we can write $f$ as, 
    $$f(x_1, \dots, x_n) = \sum_{k=1}^{2n+1} \psi_k(v_k(x))$$ 
    \\Now, as each $\phi_{k_i}$ is $C^2$, we can say, 
    $$\left| \frac{\partial v_k}{\partial x_j} \right| = \left| \phi_{k_j}'(x_j) \right| \leq B$$ 
    for some constant $B$. The partial derivatives of $f$ with respect to $x_j$ is:
    $$\frac{\partial f}{\partial x_j} = \sum_{k=1}^{2n+1} \psi_k'\left( v_k(x) \right) \cdot \phi_{k_j}'(x_j)$$ 
    where both $\psi_k'$ and $\phi_{k_j}'$ are continuous and bounded. Differentiating again, the second partials exist and are bounded by the combinations of the spline second order. Therefore, $\nabla f$ is Lipschitz continuous -- that is $f$ is $\beta$-smooth \cite{rudin1976principles}.
    
\end{proof}

\begin{proposition}\label{prop 2}
    Composition of smooth functions in KAN preserves smoothness.
\end{proposition}
\begin{proof}
    The input $x_i$ is summed at every neuron in a KAN structure, and subsequently a univariate spline $\psi(x)$ is applied. Let the sum of the inner functions be denoted as $v$. Both $\phi$, and $\psi$ are twice continuously differentiable with bounded second order derivatives. Then, we can write, $|\phi_{k_i}''| \leq \beta_1$, $|\psi_{k}''| \leq \beta_2$, and the derivative of inner sum $v'(x)$ is bounded by some constant $B>0$.
    Using chain rule, the second derivative of the composite function $f(x) = \psi_k(v(x))$ can be expressed as,
    $$f''(x)=\psi_k''(v(x))\cdot [v'(x)]^2 + \psi_k'(v(x)) \cdot v''(x)$$
    Based on the above statement, 
        $$|f''(x)|\le \beta_2 B^2 + |\psi_k'(v(x))|\cdot \beta_v$$
    where, $\beta_v$ is an upper bound on $v''(x)$. Thus each composition $\psi_k(\sum_i \phi_{k_i}(x_i))$ is $\beta$-smooth. More generally, the entire KAN, $f$, is a nested composition of such units. Hence, if each layer's activation function is twice continuously differentiable, then the final composed function also stays $C^2$ smooth \cite{deboor1978splines}.

\end{proof}

\subsection{Proposed Fairness Analysis for KAN-based Adversarial Debiasing}
In this appendix section, we discuss a supplementary special case for analyzing the fairness bounds in an adversarial debiasing setup. Although this is not a part of our main algorithmic contributions, but serves to provide additional insight into the behavior of the framework under the following conditions. We constraint the adversary model, $g$, through gradient penalty in order to satisfy and hold the following case \cite{arjovsky2017wasserstein}. 
\\Let $(X, d_\mathcal{X})$ and $(Y, d_\mathcal{Y})$ be the metric spaces. Given, the KAN classifier function $f:X \to Y$ is $L$-Lipschitz, i.e. $d_\mathcal{Y}(f(x), f(y)) \leq L \cdot d_\mathcal{X}(x,y)$ for all $x,y \in X$. 

\begin{definition}[Wasserstein-1 distance]\label{def 1}
For all probability distributions $\mu$ and $\nu$ on a metric space $X$, the Wasserstein-1 distance between $\mu$ and $\nu$ is defined as
\[
W_1(\mu,\nu)
= \sup_{\|\varphi\|_{\mathrm{Lip}} \le 1}
\left( \mathbb{E}_{x \sim \mu}[\varphi(x)] - \mathbb{E}_{x \sim \nu}[\varphi(x)] \right),
\]
where the supremum is taken over all $1$-Lipschitz functions $\varphi : X \to \mathbb{R}$ \cite{villani2008optimal}.
\end{definition}

\begin{proposition}\label{prop 3}

    Let $\mu$, $\nu$ be two distributions on $X$, and $f:X \to Y$ be $L$-Lipschitz, then,
    $$W_1(f_\# \mu, f_\# \nu) \leq L\cdot W_1(\mu, \nu)$$
\end{proposition}


where, $f_\# \mu$ and $f_\#\nu$ are the pushforward measures. 
\begin{proof}
    For a 1-Lipschitz function $g$ on $Y$, the composition $g \circ f$ is L-Lipschitz on $X$ as,
    \begin{align}
    g(f(x)) - g(f(y)) &\leq \|g\|_{\mathrm{Lip}=1} \, d_\mathcal{Y}(f(x), f(y)) \notag \\
                      &\leq 1 \cdot L \, d_\mathcal{X}(x, y) \notag \\
                      &\leq L \, d_\mathcal{X}(x, y)
    \end{align}
    Let $\varphi = \frac{1}{L}g\circ f$, or technically speaking $\varphi = \frac{1}{L}g(f(x))$. Then $\varphi$ is 1-Lipschitz on $X$ and by Definition \ref{def 1}, 
    $$\mathbb{E}_{\mu}[g\circ f]- \mathbb{E}_{\nu}[g\circ f] = L(\mathbb{E}_{\mu}[\varphi]- \mathbb{E}_{\nu}[\varphi]) \leq LW_1(\mu, \nu)$$
    Next, taking the supremum over all 1-Lipschitz $g$ on $Y$ yields,
    $$W_1(f_\#\mu, f_\#\nu) = \text{sup}_{\|g\|_{\text{Lip}\leq 1}}(\mathbb{E}_{\mu}[g\circ f]- \mathbb{E}_{\nu}[g\circ f]) \leq LW_1(\mu,\nu)$$

\end{proof}

\begin{corollary}

Let $Z \in \{ 0, 1\}$ be a sensitive attribute, then,
$\mu := \mathbb{P}(X|Z=0)$, $\nu := \mathbb{P}(X|Z=1)$ are the conditional distributions of the input features $X$ given $Z=0$ and $Z=1$ respectively. For any KAN, $f$, with Lipschitz constant $L_f$, in order to obtain a 1-Lipschitz adversary $g$, the empirical fairness objective for the adversary model can be stated as,
$$L_{adv}(g,f) := \mathbb{E}[g(f(X))|Z=0] - \mathbb{E}[g(f(X))|Z=1].$$
By Proposition \ref{prop 3}, this difference is bounded: $|L_{adv}(g,f)| \leq L_f \cdot W_1(\mu, \nu)$. Hence, maximizing $L_{adv}(g,f)$ over 1-Lipschitz $g$ yields $W_1(f_\#\mu, f_\#\nu)$.
\end{corollary}


\subsection{Smoothness-Dimension Tradeoff in KAN}
In \cite{samadi2024smooth} it states that if the target KAN function, $f \in C^k(\mathbb{R}^n)$ is smooth, then approximating it via a finite composition of smooth, lower-dimensional component functions, $g \in C^{k'}(\mathbb{R}^{n'})$, may not be sufficient for universal approximation. The limitation stems from the Vitushkin's theorem, that implies that such a representation may require the component functions to become oscillatory in order to fully capture the complexity of $f$. Formally, the univariate mappings need to satisfy, $k'n' \leq kn$. This helps us understand the oscillatory loss function convergence behavior in the conducted ablation study on increasing the $k$ value, owing to the dimensionality-smoothness tradeoff. However, in our experiments, we prioritize training stability and convergence by restricting the spline functions to remain sufficiently smooth ($C^2$), and focus on models with spline knot complexity, $k$, of 3. This may limit expressiveness of KANs, but empirically aligns with the adversarial learning framework.

\section{Additional Experimental Results}
All the experiments, excluding those involving $D_{(2)}$ datasets, have been conducted on MacBook Air with Apple M3 chip, and 24GB unified memory, running macOS. Large scale training on the $D_{(2)}$ datasets were performed on a remote Linux server with four NVIDIA RTX A6000 GPUs (49GB each), using CUDA 12.4. We have used Python-based NumPy and PyTorch packages to train the models. The hyperparameter configurations used for all experiments conducted on the proposed models are summarized in Table \ref{tab:hyperparameters}.

\begin{table*}[ht]
\centering
\renewcommand{\arraystretch}{1.3}
\begin{tabular}{l|l|c|c|c|c|c|c|c|c|c}
\toprule
\textbf{Model} & \textbf{Optimizer} & \textbf{$\text{LR}_{clf}$} & \textbf{$\text{LR}_{adv}$} & \textbf{$\text{Width}_{clf}$} & \textbf{$\text{Width}_{adv}$} & \textbf{$\text{L}_2$ Penalty} & \textbf{$\text{L}_1$ Penalty} & \textbf{$\lambda_{(1)}$} & \textbf{$\lambda_{(2)}$} & \textbf{$\eta$} \\
\midrule
 & Adam   & 0.01 & 0.1 & [140, 64, 32, 1] & [1, 32, 2] & 0.001 & 0.001 & 0.21 & 0.32 & 0.04 \\
$\text{K}_{(1)}$ & OAdam  & 0.01 & 0.1 & [140, 32, 8, 1] & [1, 32, 2] & 0.0001 & 0.01 & 0.24 & 0.33 & 0.04 \\
 & ADOPT  & 0.01 & 0.1 & [140, 64, 32, 1] & [1, 32, 2] & 0.001 & 0.01 & 0.25 & 0.37 & 0.04\\

\midrule
 & Adam  & 0.01 & 0.1 & [140, 64, 8, 1] & [1, 32, 2] & 0.0001 & 0.01 & 0.30 & 0.45 & 0.04 \\
$\text{K}_{(2)}$ & OAdam  & 0.01 & 0.01 & [140, 32, 8, 1] & [1, 32, 2] & 0.0001 & 0.01 & 0.31 & 0.43 & 0.04\\
 & ADOPT  & 0.01 & 0.01 & [140, 32, 8, 1] & [1, 32, 2] & 0.0001 & 0.01 & 0.36 & 0.49 & 0.04 \\

\bottomrule
\end{tabular}
\label{tab:hyperparameters}
\captionsetup{font=small, labelfont=bf}
\caption{Hyperparameter list for proposed KAN-based adversarial debiasing frameworks, $K_{(1)}$ (trained on $D_{(1)}$ dataset), and $K_{(2)}$ (trained on $D_{(2)}$ dataset) with an adaptive $\lambda$-update policy under three different optimizers (Adam, OAdam, and ADOPT). Width denotes the layer sizes for the classifier/adversary; $\text{LR}_{clf}$ $\&$ $\text{LR}_{adv}$ are learning rates; $L_2$ $\&$ $L_1$ penalty are regularization coefficients; $\lambda_1$ $\&$ $\lambda_2$ are the final fairness penalties obtained for Low Inc. status $\&$ First-Gen. flag sensitive features respectively; $\eta$ is the fairness learning rate.}
\end{table*}

\subsection{Justification of Optimizers}
For simulations, we compare the proposed model performance across three different types of optimizers - Adam, Optimistic Adam (OAdam), and ADOPT. This comparison is motivated by the understanding that the optimizer selection significantly impacts convergence and stability in adversarial training setups \cite{liu2020loss}. Adam is included as it is a strong baseline for general NN training. Whereas, OAdam modifies Adam with an optimistic mirror-descent step, specifically designed for adversarial problems. Similarly, ADOPT tries to achieve optimal convergence rate adaptively, making it suitable for adversarial training. By conducting an ablation study we explore how different optimizers influence training dynamics and fairness outcomes.

\subsection{Justification of Sensitive Attributes}

To understand the effectiveness of the debiasing approach, we examine both the initial and the final stages of model training. Through these experiments, we compare the performance of pre-trained and post-trained models. The \textit{pre-trained models} correspond to the initial stage of training, where the classifier and adversary are trained independently prior to the debiasing. Here, the classifier is pre-trained to predict the target variable, while the adversary is pre-trained to predict the sensitive attributes from the classifier's final prediction. It serves as a baseline, allowing us to observe the inherent bias before incorporating the debiasing algorithm. Subsequently, the \textit{post-trained models} correspond to the comprehensive adversarial debiasing training process using the adaptive $\lambda$ policy mechanism. In Figure \ref{fig:kdeplots}, we showcase the Kernel Density Estimation (KDE) Plots for both pre-training, and post-training version of the proposed adversarial framework. In the KDE plots, the effective debiasing is demonstrated by increased overlapping between the curves, indicating reduced influence of the sensitive features on the model predictions.
Table \ref{tab:pre-post-perf} showcases the fairness metrics, $p\%$-Rule and Demographic Parity (DP), for both pre-trained and post-trained models for the KAN-based adversarial framework using $k=3$ spline knot complexity. This demonstrates the disparity in outcomes with respect to the two binary sensitive attributes (Low Inc. status, and First-Gen. flag) in the pre-training phase. It further highlights the post-training improvements in fairness metrics, thereby validating the effectiveness of the proposed adversarial learning framework.

\begin{table*}[ht]
\centering
\renewcommand{\arraystretch}{1.3}
\begin{tabular}{l | c c c c | c c c c}
\toprule
\textbf{Optimizer} &
\multicolumn{4}{c|}{\textbf{Pre-Training}} &
\multicolumn{4}{c}{\textbf{Post-Training}} \\
\cmidrule(lr){2-5} \cmidrule(lr){6-9}
& $p\%-\text{Rule}_{(1)}$ & $p\%-\text{Rule}_{(2)}$  & $\text{DP}_{(1)}$ & $\text{DP}_{(2)}$
& $p\%-\text{Rule}_{(1)}$  & $p\%-\text{Rule}_{(2)}$  & $\text{DP}_{(1)}$ & $\text{DP}_{(2)}$ \\
\midrule
Adam   & 29.67 & 26.82 & 0.25 & 0.26 & 85.79 & 89.65 & 0.028 & 0.035 \\
OAdam  & 31.58 & 26.38 & 0.228 & 0.244 & 90.39 & 92.99 & 0.072 & 0.084 \\
ADOPT  & 21.22 & 19.21 & 0.288 & 0.304 & 89.21 & 90.92 & 0.047 & 0.055 \\
\bottomrule
\end{tabular}
\label{tab:pre-post-perf}
\captionsetup{font=small, labelfont=bf}
\caption{Hyperparameter list for proposed KAN-based adversarial debiasing frameworks, $K_{(1)}$ (trained on $D_{(1)}$ dataset), and $K_{(2)}$ (trained on $D_{(2)}$ dataset) with an adaptive $\lambda$-update policy under three different optimizers (Adam, OAdam, and ADOPT). Width denotes the layer sizes for the classifier/adversary; $\text{LR}_{clf}$ $\&$ $\text{LR}_{adv}$ are learning rates; $L_2$ $\&$ $L_1$ penalty are regularization coefficients; $\lambda_1$ $\&$ $\lambda_2$ are the final fairness penalties obtained for Low Inc. status $\&$ First-Gen. flag sensitive features respectively; $\eta$ is the fairness learning rate.}
\end{table*}

\begin{figure*} 
    \centering
        \captionsetup{font=small, labelfont=bf, justification=centering} 
        \makebox[0.1cm]{}\includegraphics[width=0.98\textwidth]{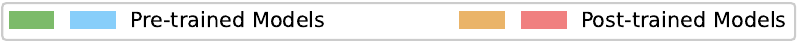}
    \hfill
        \subfloat[Training with Adam optimizer]{%
            \begin{subfigure}{0.5\linewidth}
                \centering
                \includegraphics[width=\linewidth]{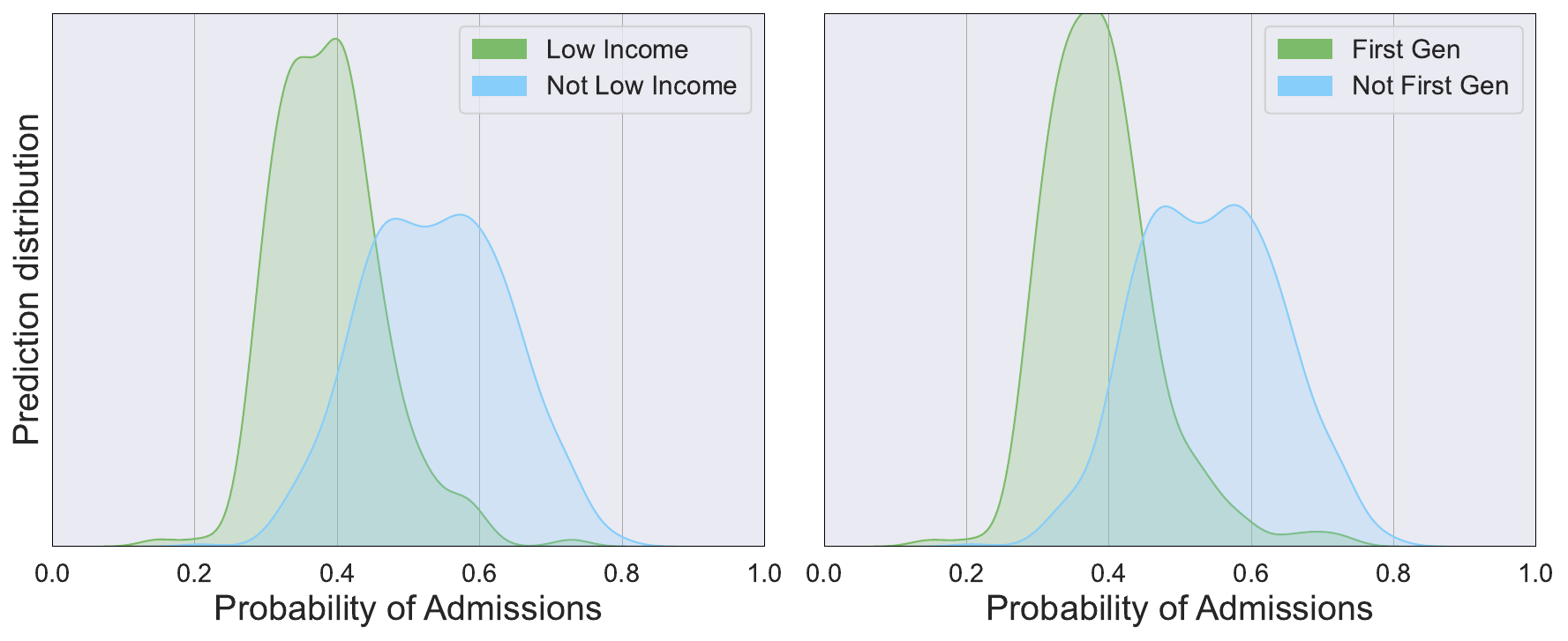}
            \end{subfigure}
            \begin{subfigure}{0.5\linewidth}
                \centering
                \includegraphics[width=\linewidth]{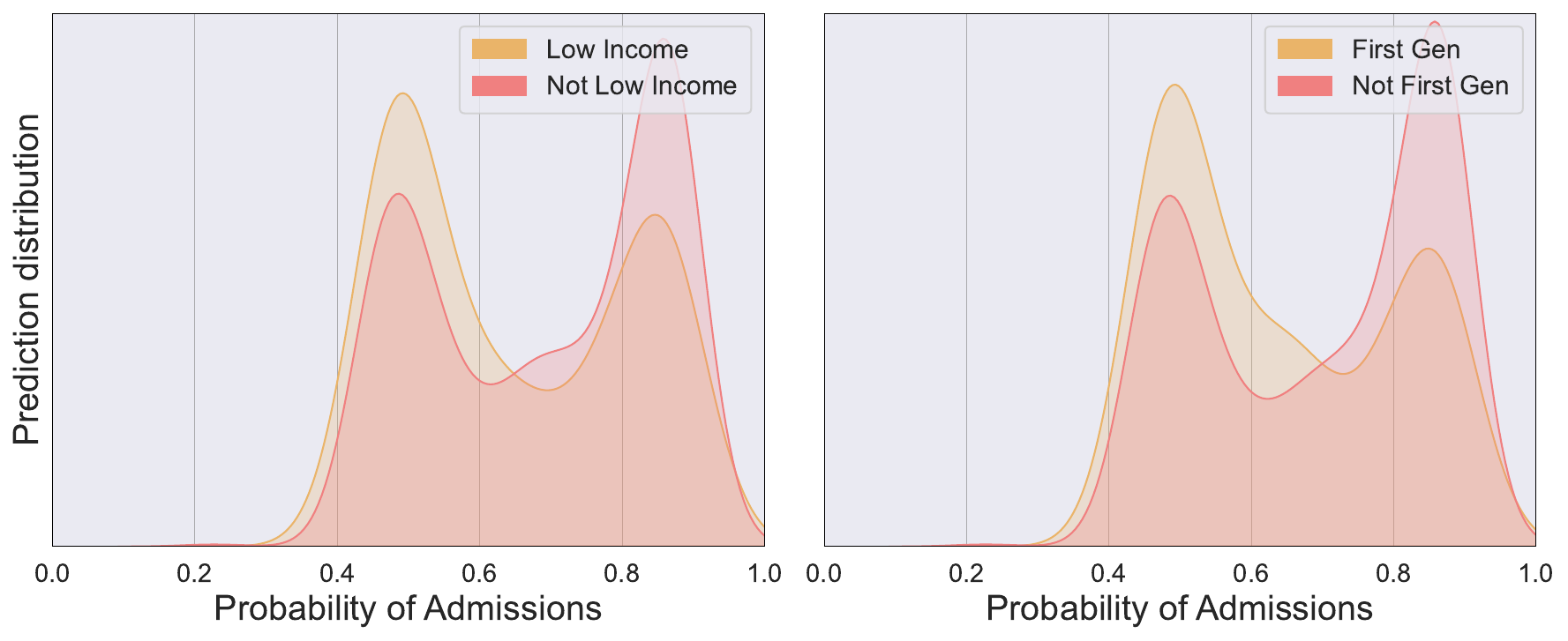}
            \end{subfigure}
        }
    \hfill
        \subfloat[Training with OAdam optimizer]{%
            \begin{subfigure}{0.5\linewidth}
                \centering
                \includegraphics[width=\linewidth]{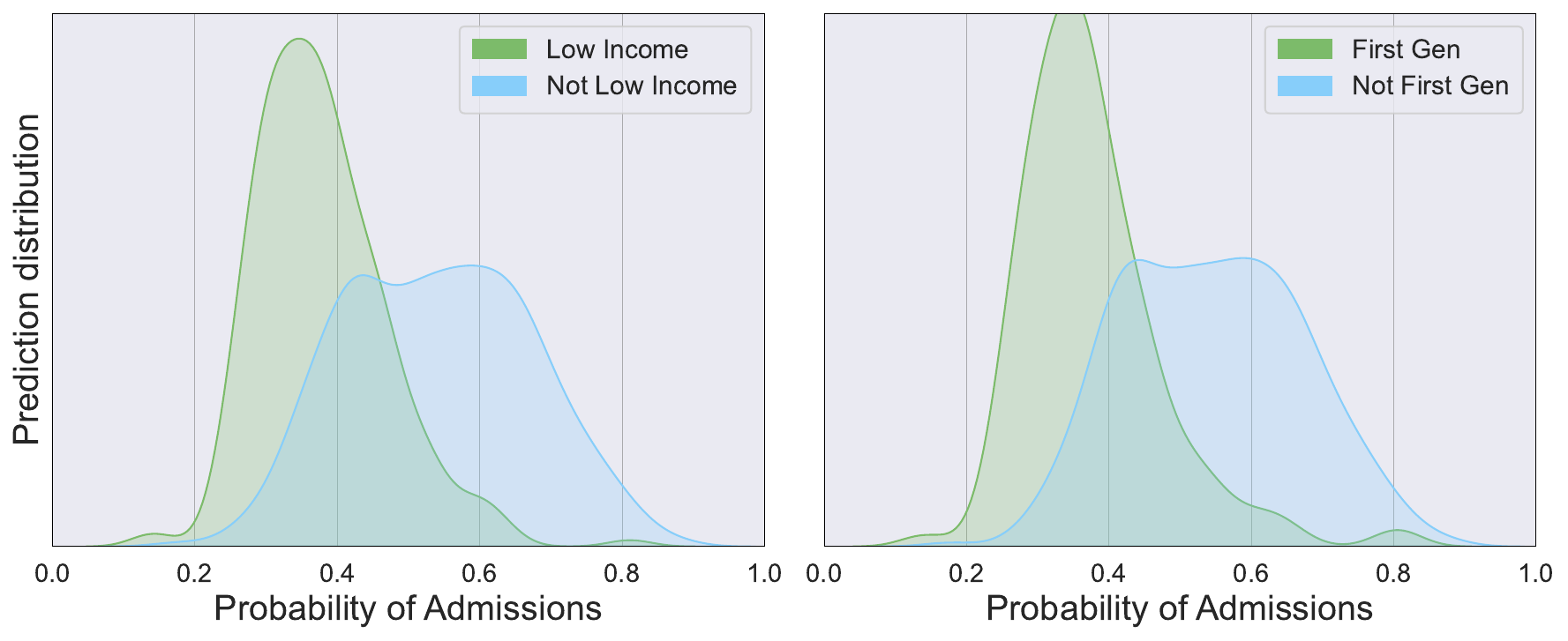}
            \end{subfigure}
            \begin{subfigure}{0.5\linewidth}
                \centering
                \includegraphics[width=\linewidth]{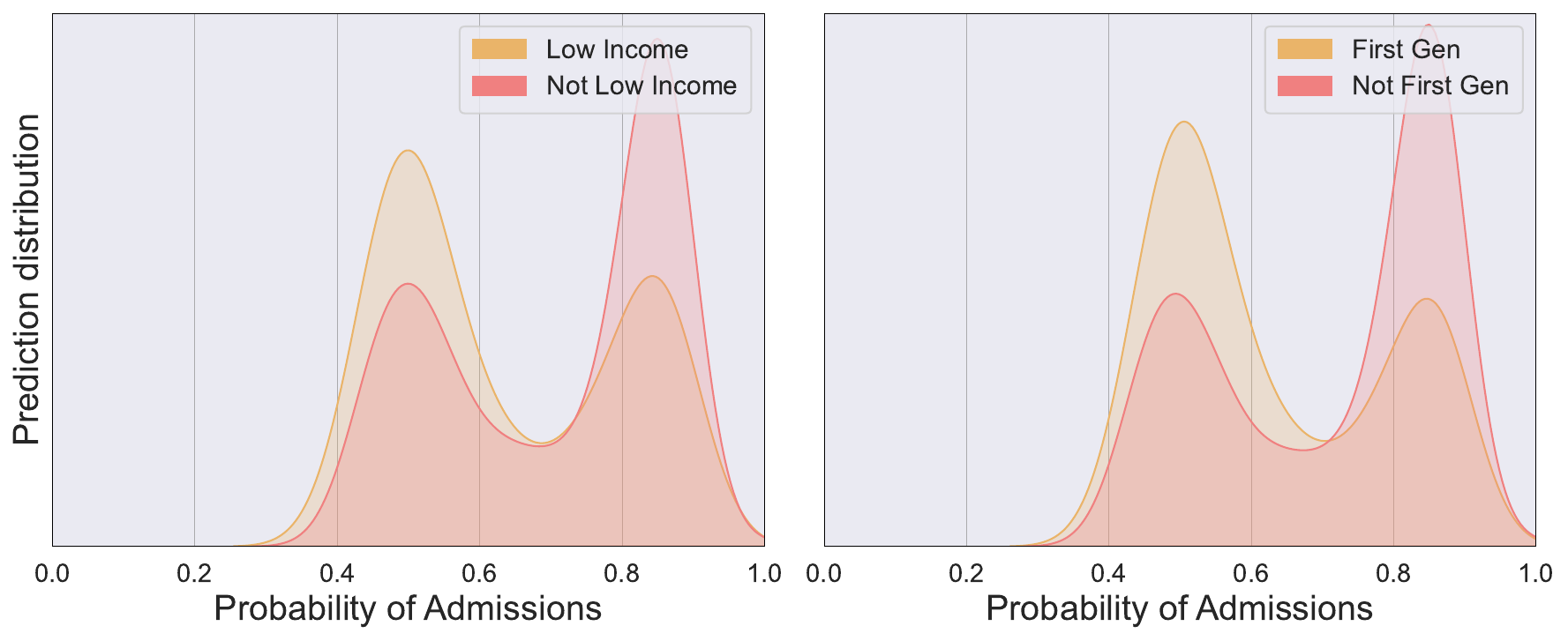}
            \end{subfigure}
        }
    \hfill
        \subfloat[Training with ADOPT optimizer]{%
            \begin{subfigure}{0.5\linewidth}
                \centering
                \includegraphics[width=\linewidth]{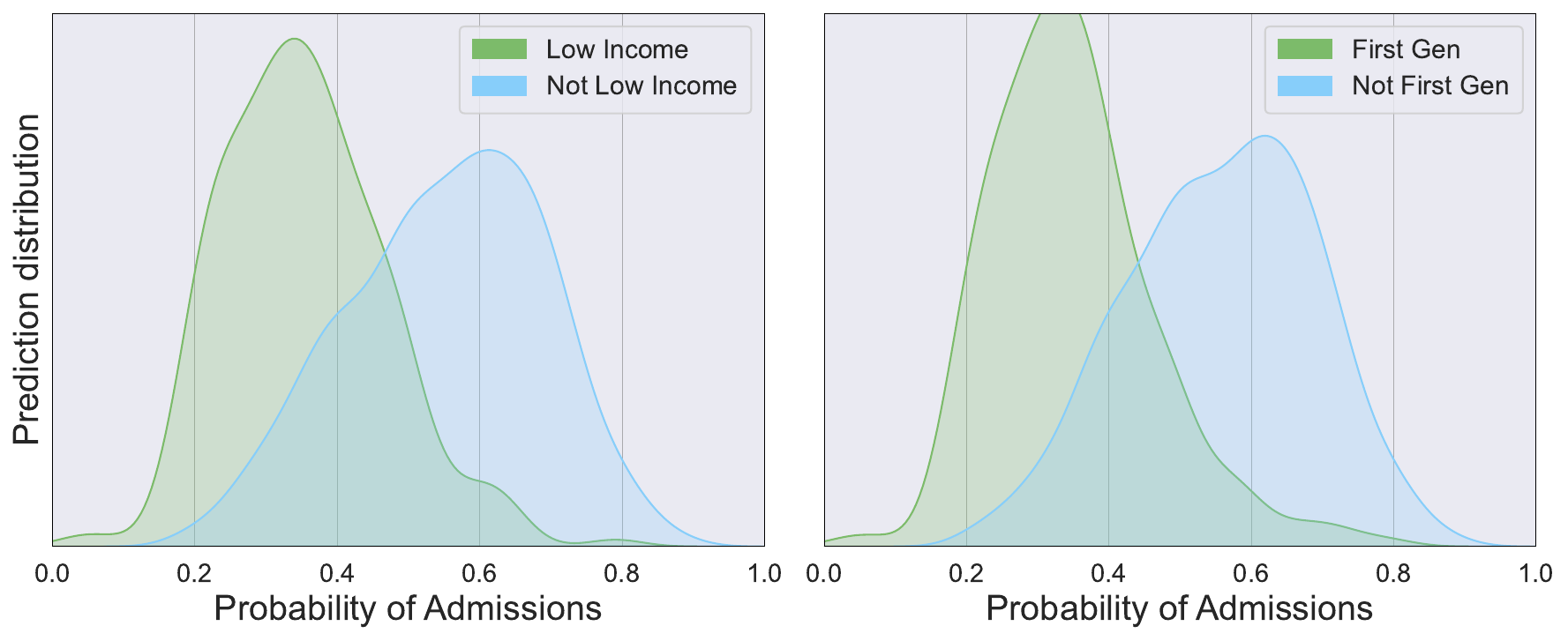}
            \end{subfigure}
            \begin{subfigure}{0.5\linewidth}
                \centering
                \includegraphics[width=\linewidth]{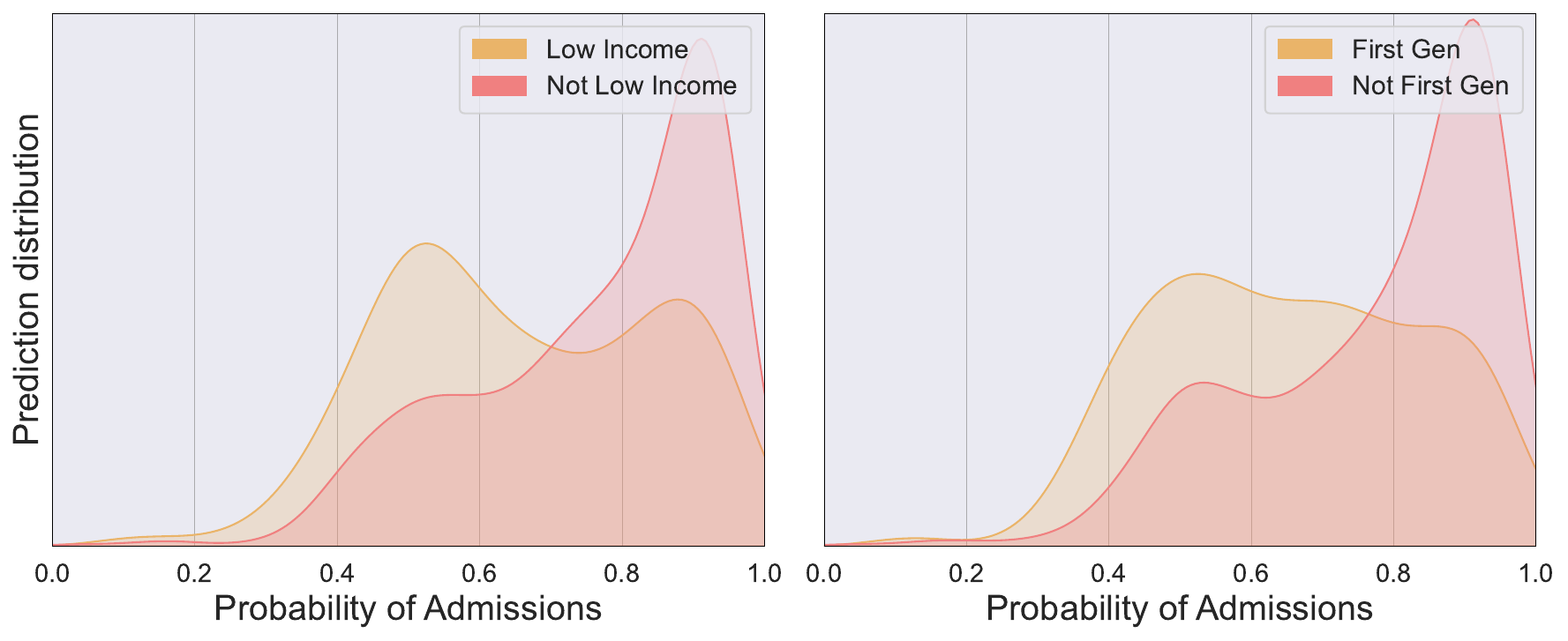}
            \end{subfigure}
        }

  \caption{The Kernel Density Estimate (KDE) plots for KAN-based adversarial debiasing model, using adaptive $\lambda$ policy. It uses the $k=3$ spline knot complexity of KAN models (for both classifier and adversary models), and is trained on $D_{(1)}$ dataset. It depicts the pre-trained and post-trained models, highlighting the bias present in the model predictions.}
  \label{fig:kdeplots} 
\end{figure*}

\end{document}